\documentclass[10pt,journal,compsoc]{IEEEtran}
\IEEEoverridecommandlockouts
\usepackage{cite}
\usepackage{hyperref}
\usepackage{amsmath,amssymb,amsfonts,amsthm,mathtools}
\allowdisplaybreaks[4]
\newtheorem{theorem}{Theorem}
\newtheorem{assumption}{Assumption}
\newtheorem{lemma}{Lemma}
\newtheorem{remark}{Remark}
\newtheorem{corollary}{Corollary}
\newtheorem*{proof*}{Proof}
\usepackage{algorithmic,algorithm}
\usepackage{graphicx}
\usepackage{textcomp}
\usepackage{xcolor}
\usepackage{multirow}
\usepackage{ulem}
\usepackage{booktabs}
\usepackage{url}
\def\BibTeX{{\rm B\kern-.05em{\sc i\kern-.025em b}\kern-.08em
    T\kern-.1667em\lower.7ex\hbox{E}\kern-.125emX}}
\begin{document}

\title{
How to Collaborate: Towards Maximizing the Generalization Performance in Cross-Silo Federated Learning
}

\author{Yuchang~Sun,~\IEEEmembership{Graduate Student Member,~IEEE},
Marios~Kountouris,~\IEEEmembership{Fellow,~IEEE},
and Jun~Zhang,~\IEEEmembership{Fellow,~IEEE}

\thanks{
        Y. Sun and J. Zhang are with the Department of Electronic and Computer Engineering, The Hong Kong University of Science and Technology, Hong Kong (e-mail: yuchang.sun@connect.ust.hk;eejzhang@ust.hk).
        M. Kountouris is with the Andalusian Research Institute in Data Science and Computational Intelligence (DaSCI), Department of Computer Science and Artificial Intelligence, University of Granada, Spain, and with EURECOM, France (e-mail: mariosk@ugr.es). The work of J. Zhang was supported by the Hong Kong Research Grants Council under the Areas of Excellence scheme grant AoE/E-601/22-R and NSFC/RGC Collaborative Research Scheme grant CRS\_HKUST603/22. The work of M. Kountouris was supported by the European Research Council (ERC) under the European Union’s Horizon 2020 Research and Innovation Programme (Grant agreement No. 101003431).
        (Corresponding author: Jun Zhang)
        }
        
}

\IEEEtitleabstractindextext{%
\begin{abstract}
Federated learning (FL) has attracted vivid attention as a privacy-preserving distributed learning framework. In this work, we focus on cross-silo FL, where clients become the model owners after training and are only concerned about the model’s generalization performance on their local data. Due to the data heterogeneity issue, asking all the clients to join a single FL training process may result in model performance degradation. To investigate the effectiveness of collaboration, we first derive a generalization bound for each client when collaborating with others or when training independently. We show that the generalization performance of a client can be improved by collaborating with other clients that have more training data and similar data distributions. Our analysis allows us to formulate a client utility maximization problem by partitioning clients into multiple collaborating groups. A \underline{h}ierarchical \underline{c}lustering-based \underline{c}ollaborative \underline{t}raining (HCCT) scheme is then proposed, which does not need to fix in advance the number of groups. We further analyze the convergence of HCCT for general non-convex loss functions which unveils the effect of data similarity among clients. Extensive simulations show that HCCT achieves better generalization performance than baseline schemes, whereas it degenerates to independent training and conventional FL in specific scenarios.
\end{abstract}

}

\maketitle

\section{Introduction}\label{sec:introduction}

\IEEEPARstart{F}{ederated} learning (FL), a distributed model training paradigm, has recently received much attention due to its benefits in preserving data privacy \cite{fedavg,app1,ni1}.
In an FL system, a central server coordinates multiple clients with private data for training deep learning (DL) models.
Specifically, participating clients collaborate to find a global model that achieves satisfactory performance on all clients' data by sharing model updates instead of training data.
During the FL training process, clients receive a global model from the server and optimize it based on the local data. After local training, clients upload the accumulated model updates to the server for aggregation. The server in turn updates the global model using the aggregated model updates and sends it back to clients for the next epoch training.

Depending on the types of clients, FL can be classified into cross-device and cross-silo settings \cite{kairouz2021advances,mimic}. 
In this work, we focus on cross-silo FL \cite{huang2022cross,incentive1,zhang2020batchcrypt} where clients are typically companies or organizations (e.g., banks or hospitals) and become the model owners after FL training.
These clients, which normally have sufficient computation resources, are expected to be continuously available throughout the training process. However, they are usually self-interested and are only concerned about the generalization performance on their local data \cite{huang2021personalized,yu2020salvaging}.
Specifically, each client wants to learn a model that can generalize well on its local distribution.
In practice, some powerful clients with sufficient training data can achieve a good generalization performance with independent training, thus having limited or no incentive to collaborate with others.
Meanwhile, the data among clients are naturally not independent and identically distributed (non-IID); this is known to severely degrade the training performance of the global model \cite{li2022silos}.
Hence, participating in FL may be a suboptimal choice for some clients.
On the other hand, clients with a limited amount of data may prefer collaborating with others to benefit from their training data. Nevertheless, they also tend to exclude the diverged gradients uploaded by clients with disparate data distribution. As such, training a single global model by incorporating all clients cannot suit their requirements. Given the above considerations, our paper aims to address the challenge of how to collaborate among clients under the FL framework, which demands a flexible collaboration plan.


There have been a few attempts to improve the utility of collaboration in FL.
A recent work \cite{cho2022federate} proposed to only incorporate the clients that find the global model beneficial in collaborative training while allowing others to train independently.
Nevertheless, those independent clients, which have different data distributions from most clients, may not enjoy satisfactory performance by making a binary decision of whether to participate in FL or not.
To remedy the deficiency of a single global model, some works \cite{ifca,li2021soft,kim2021dynamic} introduce multiple global models to FL system.
Specifically, they cluster all clients into several groups, and the clients in the same group cooperatively train a global model.
The clustering criterion is tailored to minimize the training error of clients, which, however, is not sufficient to guarantee good performance on unseen or new data \cite{yuan2021we}.
To enhance clients' generalization performance, the design of an appropriate metric for partitioning clients remains an unresolved challenge.
Besides, determining the number of groups without additional knowledge of local data distribution is a notoriously challenging problem.
As a more general consideration, some studies resort to game theory to formulate the client partition problem as a hedonic game without fixing the number of groups \cite{hedonic1,hedonic2,hedonic3,hedonic4,9842363}.
However, these works focus on the theoretical aspects of this game, e.g., stability and equilibrium, and they do not provide any practical and efficient algorithm to solve the problem.

In this work, we investigate how to design the collaboration pattern of clients so as to maximize their generalization performance in a cross-silo FL system.
Our contributions are summarized as follows:
\begin{itemize}
    \item By collaborating, clients can benefit from more training data, but may also suffer from data heterogeneity. To study in detail this phenomenon, we analyze the local test error of a client when it joins any group or trains a model independently.
    Our analytical results unveil that the generalization performance of a client can be improved by involving more training data samples yet excluding the collaborators with diverged data distributions. 
    \item Motivated by this analysis, we define the utility of each client as an upper bound of the test error, where the divergence of data distribution is approximated by the gradient distance.
    Then, we formulate a client utility maximization problem by designing the collaboration among clients.
    To solve the problem efficiently, we propose a \underline{h}ierarchical \underline{c}lustering-based \underline{c}ollaborative \underline{t}raining scheme, coined \emph{HCCT}, which identifies the proper collaboration patterns for clients during training.
    It is worth noting that the number of groups in HCCT is automatically determined without requiring additional tuning. Considering the computational complexity of evaluating client similarity, we further propose an efficient implementation of HCCT.
    \item We prove the convergence of the HCCT scheme for general non-convex loss functions by analyzing the sum of local gradients of clients.
    With a more precise assumption on gradient similarity, we can characterize the effect of client grouping on the convergence performance in cross-silo FL.
    \item Finally, we evaluate the proposed HCCT scheme via simulations in different training tasks and datasets. 
    Our experimental results show that HCCT achieves better generalization performance than the baselines and can adapt to various scenarios. Furthermore, the ablation studies provide guidelines on the selection of hyperparameters in HCCT.
\end{itemize}

\textbf{Organization.}
The rest of the paper is organized as follows. Section \ref{sec:related} introduces the related work. In Section \ref{sec:system}, we describe the system model and a motivating example. Then, we analyze the generalization performance of clients and formulate the client utility maximization problem in Section \ref{sec:problem}. In Section \ref{sec:design}, we propose HCCT to optimize the collaboration pattern, establish its convergence in Section \ref{sec:theory}, and evaluate it via simulations in Section \ref{sec:simulation}. We conclude this paper in Section \ref{sec:conclusion}.

\section{Related Works}\label{sec:related} 
In this section, we introduce related approaches that consider the problem of collaboration patterns in FL with self-interested clients.
We note that personalization techniques, e.g., \cite{huang2021personalized,yu2020salvaging,tan2022towards,arivazhagan2019federated,fallah2020personalized}, can be seen as orthogonal to our study, since these techniques can be employed to further enhance the model performance at clients after the collaboration design.

\textbf{Fair collaboration in single-model FL.} 
Considering the client requirements, Cho \textit{et al.} \cite{cho2022federate} proposed a training scheme named \emph{MAXFL}, where each client participates in FL only if it finds the global model appealing, namely, if the global model yields a smaller training loss than a self-defined threshold.
This threshold is defined as the difference between the current global model and a local model optimized in an independent training process, which, however, introduces additional computation costs.
More importantly, the binary decision used in this work excludes the possibility that those independent clients may formulate a group to further improve their generalization performance.
Another stream of work \cite{huang2020fairness,mohri2019agnostic} aims to achieve fair model performance among clients by modifying the global training objective.
For example, \cite{huang2020fairness} proposed to give a higher weight to the model updates from clients with lower training accuracy.
Nevertheless, these attempts still focus on the setup with a single global model, which limits the performance improvement for most clients.

\textbf{Clustered federated learning.}
To satisfy different local objectives, clustered FL generates multiple global models by clustering clients into several groups \cite{ifca,li2021soft,kim2021dynamic}.
For example, an iterative clustering algorithm for FL named \emph{IFCA} was proposed in \cite{ifca}, where client group identity estimation and federated model training are carried out alternatively.
To estimate the group identity, clients are required to evaluate each global model on their local data to find the one achieving the minimal training loss.
However, this process incurs additional training cost that increases linearly with the number of groups.
More importantly, the clustering criterion in these works focuses on minimizing the training loss \cite{ifca,li2021soft,9174890} but fails to reflect the generalization performance on the unseen data.
Furthermore, selecting the number of groups is known to be challenging in these algorithms as the server does not know local data distribution.

\textbf{Game theoretic approaches.}
As a more general formulation, one branch of studies \cite{hedonic1,hedonic2,hedonic3,hedonic4} views the collaboration design of clients in FL as a hedonic game without fixing the number of groups and arranges the client partition to minimize their training errors.
These works are closely related to our consideration, but they mainly focus on the theoretical aspects by analyzing the optimality, stability, and equilibrium of solutions.
Although an optimal client partition was proposed in \cite{hedonic1}, it is restricted to the mean estimation problem. A practical algorithm that can efficiently arrange client partitions in general DL setups is still lacking.
Moreover, other game theory-based works \cite{incentive1,incentive2,incentive3,incentive4} assume that clients are reluctant to collaborate due to concerns of limited computation resources or privacy leakage.
Therefore, the central server, whose objective is to obtain a satisfactory global model, gives clients sufficient rewards to encourage their participation in FL.
Nevertheless, these works focus on improving a single global model and may become suboptimal in terms of local test error since they ignore the performance requirements of clients.
In contrast to previous work, we propose an efficient algorithm to design the collaboration pattern such that all clients can achieve satisfactory generalization performance.

\section{Preliminaries}\label{sec:system}
In this section, we introduce the cross-silo FL system, followed by a motivating example.

\subsection{Cross-Silo FL}\label{sec:3A}

We consider a cross-silo FL system which consists of a central server and a set of $N$ clients denoted by $\mathcal{N}=\{1,2,\dots, N\}$.
Each client $i \in \mathcal{N}$ has a training dataset $\mathcal{D}_{i}^{tr}$ which includes $D_i \triangleq |\mathcal{D}_{i}^{tr}|$ data samples and follows data distribution $\mathcal{P}_{i}^{tr}$.
Denote the feature and label of any data sample by $\mathbf{x}\in\mathcal{X}$ and $y\in \mathcal{Y}$, respectively.
Clients aim to learn a prediction function $h(\cdot;\mathbf{w}): \mathcal{X}\rightarrow \mathcal{Y}$, which is characterized by model $\mathbf{w}\in\mathbb{R}^M$ with $M$ trainable parameters, to minimize the prediction error on the test data $\mathcal{D}_{i}^{te}$, i.e.,
\begin{equation}
    \min_{\mathbf{w}\in\mathbb{R}^M} \epsilon_{i}(h) \triangleq \mathbb{E}_{(\mathbf{x}, y) \sim \mathcal{P}_{i}^{te}}[l(h(\mathbf{x}; \mathbf{w}),y)].
\end{equation}
Here $l(h(\mathbf{x}; \mathbf{w}),y)$ is the loss function (e.g., the categorical cross-entropy) computed between the predicted label $h(\mathbf{x}; \mathbf{w})$ and the ground-truth label $y$, and $\mathcal{P}_{i}^{te}$ denotes the distribution of test data.
Although the test data are typically unknown during training, it is commonly assumed that the training data and test data are generated from the same underlying distribution \cite{long2019generalization}, i.e.,
\begin{equation}
    \mathcal{P}_{i}^{te} = \mathcal{P}_{i}^{tr} = \mathcal{P}_{i}, \forall i\in\mathcal{N}.
\end{equation}
Besides, the central server has no data sample and its role is solely to coordinate the clients for their collaboration.

In real-world scenarios, clients have various data distributions and different training dataset sizes \cite{li2022silos,duan2020self}.
To be specific, some clients may own a limited number of data samples and thus seek assistance from others for collaboration to improve the model performance.
To describe this scenario, without loss of generality, we assume that clients are partitioned into $K$ groups with indices $k\in[K]$.
Note that the number of groups $K$ is a priori unknown. 
Each group consists of a set of clients $\mathcal{C}_k \subset \mathcal{N}$, and a client belongs to only one group. 
With the available training data $\hat{\mathcal{D}}_{k}^{tr} \triangleq \cup_{i\in\mathcal{C}_k} \mathcal{D}_{i}^{tr}$, clients in group $k$ train a model by minimizing the following objective:
\begin{equation}
    \min_{\mathbf{w}\in\mathbb{R}^M} \epsilon_{k}(h) \triangleq \mathbb{E}_{(\mathbf{x}, y) \sim \hat{\mathcal{P}}_{k}^{tr}}[l(h(\mathbf{x}; \mathbf{w}),y)],
\end{equation}
where $\hat{\mathcal{P}}_{k}^{tr}$ denotes the distribution of training dataset $\hat{\mathcal{D}}_{k}^{tr}$, and $\hat{D}_{k}=|\hat{\mathcal{D}}_{k}^{tr}|$. 

In training epoch $t$, the training process of each group is detailed as follows.
Each client $i$ receives the global model $\hat{\mathbf{w}}_{G_i^{t}}^{t}$ from its corresponding group $G_i^{t}\in[K]$ and optimizes this model by using mini-batch stochastic gradient descent (SGD) for $Q$ steps.
Specifically, at step $q\in[Q]$ of epoch $t$, client $i$ updates the local model according to
\begin{equation}
    \mathbf{w}_{i}^{t,q+1} = \mathbf{w}_{i}^{t,q} - \eta^t \sum_{(\mathbf{x}, y)\in\mathcal{B}_i^{t,q}} \nabla_{\mathbf{w}} l( h((\mathbf{x};\mathbf{w}_{i}^{t,q}),y),
    \label{eq:sgd}
\end{equation}
where $\mathbf{w}_{i}^{t,0} = \hat{\mathbf{w}}_{G_i}^{t}$, $\eta^t$ is the learning rate, and $\mathcal{B}_i^{t,q}$ is a batch of training data randomly sampled from $\mathcal{D}_{i}^{tr}$.
Afterwards, each client summarizes the model updates as
\begin{equation}
    \mathbf{g}_i^{t} \triangleq \sum_{q=1}^{Q} \frac{1}{|\mathcal{B}_i^{t,q}|} \sum_{(\mathbf{x}, y)\in\mathcal{B}_i^{t,q}} \nabla_{\mathbf{w}} l( h((\mathbf{x};\mathbf{w}_{i}^{t,q}),y),
\end{equation}
which is then uploaded to the server.
This server then aggregates these model updates from clients in each group and updates the global model by averaging them, i.e.,
\begin{equation}
    \hat{\mathbf{w}}_{G_i^{t}}^{t+1} = \hat{\mathbf{w}}_{G_i^{t}}^{t} - \eta^t \hat{\mathbf{g}}_{G_i^{t}}^{t},
\end{equation}
where $\hat{\mathbf{g}}_{G_i^{t}}^{t} \triangleq \sum_{i\in\mathcal{C}_{G_i^{t}}} \frac{D_i}{\hat{D}_{G_i^{t}}} \mathbf{g}_i^{t}$.
We summarize the main notations in Table \ref{table:notation}.

\begin{table}[t]
\caption{Main Notations}
\label{table:notation}
\centering
\resizebox{\columnwidth}{!}{
\begin{tabular}{c|c}
\hline
\textbf{Notation} & \textbf{Meaning}       \\ \hline
$\mathcal{D}_{i}^{tr}$, $\mathcal{D}_{i}^{te}$ & Training dataset and test dataset at client $i$ \\
$\mathcal{P}_{i}^{tr}$, $\mathcal{P}_{i}^{te}$ & Distribution of training data and test data at client $i$\\
$D_{i}$ & Number of training data samples at client $i$  \\
$\hat{\mathcal{D}}_{k}^{tr}$ & Training dataset at group $k$ \\
$\hat{D}_{k}$ & Number of training data samples at group $k$ \\
$\hat{\mathcal{P}}_{k}$ & Distribution of training data at group $k$ \\
$\mathcal{C}_k$ & Client set of group $k$ \\
$G_i^t$ & The group index of $i$-th client at epoch $t$ \\
$\mathbf{w}_i^t$, $\hat{\mathbf{w}}_{k}^t$ & Initial model of client $i$ or group $k$ at epoch $t$ \\
$h(\cdot;\mathbf{w})$ & Prediction function characterized by $\mathbf{w}$ \\
$\epsilon_i(\cdot)$, $\epsilon_{k}(\cdot)$ & Expected risk function of client $i$ or group $k$ \\
$\mathbf{g}_i^{t}$ & Model update of client $i$ at epoch $t$ \\
$\hat{\mathbf{g}}_{k}^{t}$ & A weighted average of updates of clients in group $k$ \\
\hline
\end{tabular}
}
\end{table}

It is worth noting that clients have various data distributions \cite{fedavg,kairouz2021advances}, and thus some clients' data may differ significantly from others.
However, under strict privacy protection, the local data distribution is only visible to the client itself \cite{house2012consumer}.
If the process of joining a group is performed arbitrarily, clients may suffer from performance degradation due to the data heterogeneity issue \cite{li2022silos}.
To clearly demonstrate this phenomenon, we present an example in the next subsection.

\subsection{Motivating Example}\label{sec:example}
We consider a cross-silo FL system with $N=3$ clients. They collaborate to classify data samples extracted from the CIFAR-10 \cite{cifar10} dataset by training a convolutional neural network (CNN) model with four convolutional layers.
We assume that clients $1$ and $2$ have the same data distributions and client $3$ has different data distribution deviating from the two others.
To be concrete, clients $1$ and $2$ have respectively $20\%$ and $80\%$ of five classes of training data, while client $3$ has all the training data samples of the other five classes.
We simulate different collaboration patterns among three clients and evaluate their test errors, as shown in Table \ref{table:example}, where the best generalization performances are highlighted in bold while the worst ones are underlined.

\begin{table}[ht]
\caption{Test error (\%) after 10 training epochs in the CIFAR-10 example.}
\label{table:example}
\centering
\resizebox{.95\columnwidth}{!}{
\begin{tabular}{cc|c|c|c|c|c}
\hline
\multicolumn{2}{c|}{Pattern}                            & Client $1$ & Client $2$ & Client $3$ & Mean & Std.   \\ \hline
\multicolumn{2}{c|}{Independent}                        & $38.16 \pm 1.53$    &  $\mathbf{25.12 \pm 0.83}$   & $\mathbf{15.8 \pm 0.82}$    &  $26.36$   & $9.19$   \\ 
\multicolumn{2}{c|}{Global}                             & $49.90 \pm 1.31$     & $49.40 \pm 1.62$          &  \underline{$31.32 \pm 1.53$}          &  $43.05$      &  $8.33$      \\ \cline{1-2}
\multicolumn{1}{c|}{\multirow{3}{*}{Partial$^*$}} & $(1,2)$ &  $\mathbf{23.54 \pm 1.40}$       & $\mathbf{25.28 \pm 0.64}$           & $\mathbf{15.90 \pm 1.19}$           &  $\mathbf{21.58}$       & $\mathbf{4.10}$       \\  
\multicolumn{1}{c|}{}                         & $(1,3)$ & \underline{$96.14 \pm 4.81$}           &  $26.12 \pm 0.29$          & $16.21 \pm 0.67$           &  \underline{$46.15$}       &  \underline{$35.58$}      \\  
\multicolumn{1}{c|}{}                         & $(2,3)$ & $39.42 \pm 3.72$          &  \underline{$56.84 \pm 3.62$}          & $27.30 \pm 1.85$          &  $41.18$       & $12.32$       \\ \hline

\multicolumn{7}{l}{\footnotesize$^*$Two clients collaborate on the global model, while the other client trains independently.} 
\end{tabular}
}
\end{table}

From the above results, we observe that the best collaboration pattern is teaming up clients $1$ and $2$ while allowing independent training of client $3$, as shown in Fig. \ref{fig:three-client}.
In this pattern, all of them can learn models achieving the best local generalization performance. 
In comparison, with independent training, client $1$ suffers from a limited amount of training data, and a simple global training approach degrades the model performance of client $3$ due to the data heterogeneity issue.
Meanwhile, arbitrary collaboration, such as pairing clients $1$ and $3$, can severely impair the model performance because of their divergent data distributions.
This example demonstrates the importance of designing a collaboration pattern among clients to improve their generalization performance.
To quantify the effect of collaborative training, we analyze the test error of clients and then formulate a problem to maximize their generalization performance in the next section.

\begin{figure}[!t]
    \centering
    \includegraphics[width=\columnwidth]{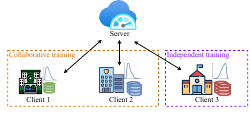}
    \caption{An example of collaboration pattern among three clients in cross-silo FL.}
    \label{fig:three-client}
\end{figure}

\section{Problem Formulation}\label{sec:problem}
\subsection{Theoretical Analysis of Generalization Performance}
Consider that client $i$ joins in a group with index $G_i \in [K]$ for model training\footnote{In the following, we omit the superscript $t$ in $G_i^t$ for simplification when not causing any confusion.}.
This group has $\hat{D}_{G_i}$ training data samples that follow a distribution of $\hat{\mathcal{P}}_{G_i}$.
When this client performs independent training, we have $\mathcal{C}_{G_i}=\{i\}$, $\hat{D}_{G_i}=D_i$, and $\hat{\mathcal{P}}_{G_i}=\mathcal{P}_i$.
In the following theorem, we quantify the test error of client $i$ by joining any group.
\begin{theorem}\label{thm1}
Consider the loss function with $\mu$-strong convexity and $L$-Lipschitz continuity.
For a prediction function $h$ and any $0 <\delta < 1$, the following holds:
\begin{align}
    \epsilon_{i}(h) \leq & \frac{4L^2}{\delta \mu \hat{D}_{G_i} } + d_1\left(\hat{\mathcal{P}}_{G_i}, \mathcal{P}_{i}\right) +\lambda,
    \label{eq:bound}
\end{align}
with probability $1-\delta$, where $d_1\left(\hat{\mathcal{P}}_{G_i}, \mathcal{P}_{i}\right)$ is the distribution divergence (e.g., $\mathcal{H}$-divergence) between $\hat{\mathcal{P}}_{G_i}$ and $\mathcal{P}_{i}$, and $\lambda = \min \left\{\mathbb{E}_{\mathcal{P}_{i}} \left[ f_{G_i}(\mathbf{x}) - f_{i}(\mathbf{x}) \right], \mathbb{E}_{\hat{\mathcal{P}}_{G_i}} \left[ f_{G_i}(\mathbf{x}) - f_{i}(\mathbf{x}) \right] \right\}$ with true labeling functions $f_{i}(\cdot)$ and $f_{G_i}(\cdot)$.
\end{theorem}
\begin{proof*}

The result is proved following \cite{ben2010theory,shalev2010learnability}.
Please refer to Appendix A.
\end{proof*}

\begin{remark}
    According to Theorem \ref{thm1}, the test error of client $i$ decreases with more training samples (i.e., a larger value of $\hat{D}_{G_i}$) or smaller distribution divergence between training data and test data (i.e., a smaller value of $d_1\left(\hat{\mathcal{P}}_{G_i}, \mathcal{P}_{i}\right)$).
\end{remark}

The above results show that any client should seek collaborators with similar training data distributions and more data samples.
However, the data distributions of clients cannot be disclosed to others due to privacy concerns.
To solve this problem, in the following subsection, we propose a method to estimate the data divergence and formulate the client utility maximization problem.

\subsection{Client Utility Maximization}
Based on the analytical results, we define the utility of client $i$ as the upper bound of its test error in Theorem \ref{thm1}.
Nevertheless, the data distribution divergence $d_1\left(\hat{\mathcal{P}}_{G_i}, \mathcal{P}_{i}\right)$ in \eqref{eq:bound} is intractable due to the unknown $\hat{\mathcal{P}}_{G_i}$ since clients do not share the data or data distribution to others.
To solve this problem, we propose to use the gradient divergence to approximate the divergence of data distribution.
This is inspired by the gradient matching methods in domain generalization \cite{shi2021gradient} and dataset condensation \cite{zhao2020dataset}, where close gradient directions indicate similar data distributions.
To be specific, define $s\left( i, G_i \right)$ as the similarity between client $i$ and its corresponding group $G_i$, where $s: \mathbb{R}^{M} \times \mathbb{R}^{M} \rightarrow \mathbb{R}$ is a similarity function.
One approach is to use cosine similarity of gradients:
\begin{equation}
    s\left( i, G_i^t \right) 
    = \cos \langle \mathbf{g}_i^t, \hat{\mathbf{g}}_{{G_i^t}}^{t} \rangle
    = \frac{\left\langle \mathbf{g}_i^t, \hat{\mathbf{g}}_{{G_i^t}}^t \right\rangle}{\left\|\mathbf{g}_i^t\right\|_2 \left\|\hat{\mathbf{g}}_{{G_i^t}}^t\right\|_2}.
\end{equation}
Thus, the utility of client $i$ is given by
\begin{equation}
    U_i(G_i^t) = U_i(\hat{D}_{G_i^t}, \hat{\mathbf{g}}_{{G_i^t}}^{t}) \triangleq - \frac{\alpha}{\hat{D}_{G_i^t}} + s\left( i, G_i^t \right) + \beta,
    \label{eq:utility}
\end{equation}
where $\alpha > 0$ is a trade-off constant, and $\beta > 0$ is a large constant ensuring non-negative utility. Such non-negativity guarantees its physical meaning of utility without impacting system design.
We note that $\alpha$ balances between training data volume and gradient similarity. If clients have sufficient data, they prefer collaborating with similar clients by adopting a small $\alpha$; otherwise, they are hungry for more data and choose a large $\alpha$.
The impact of $\alpha$ is discussed further in Section \ref{sec:effect} via simulations.
According to the definition in \eqref{eq:utility}, a client can improve its utility by involving more training data and increasing the similarity between gradients.

We now formulate the client utility maximization problem by designing client partition $\Pi$ as follows
\begin{align}
    \textbf{P1}: \max_{\Pi} & \sum_{i=1}^{N} U_i(\hat{D}_{G_i^t}, \hat{\mathbf{g}}_{G_i^t}^t), \label{eq:objective} \\ 
    \text{s.t. } & \Pi \!\triangleq\! \left\{\{\mathcal{C}_k\}_{k=1}^{K}|\mathcal{C}_k \neq \emptyset, \mathcal{C}_k \cap \mathcal{C}_{k^\prime} = \emptyset, \bigcup_{k=1}^{K} \mathcal{C}_{k} = \mathcal{N} \right\}.
    \label{eq:condition}
\end{align}
Here the constraint in \eqref{eq:condition} implies that all clients are divided into $K$ non-empty and non-overlapping groups.
Note that the number of groups $K$ is not a predetermined value but an unknown integer satisfying $1 \leq K \leq N$.

Finding an optimal solution for Problem (P1) is $\mathcal{NP}$-hard in general \cite{hedonic4}.
One straightforward approach is to search over all $B_N$ possible partitions, where $B_N$ denotes the Bell number of a set with $N$ elements.
This approach, however, introduces extremely high computational complexity.
We notice that if $\alpha=0$, Problem (P1) reduces to a $k$-means problem which partitions $N$ clients into $K$ groups to minimize the within-cluster variance, i.e., the distance between each point $\mathbf{g}_{i}^{t}$ and its corresponding center $\mathbf{g}_{G_i}^{t}$.
There are many classic methods for solving this problem, e.g., $k$-means clustering \cite{kmeans}, $k$-means++ \cite{kmeans++} and hierarchical clustering \cite{ward,grosswendt2020theoretical} algorithms.
However, the presence of the constraint $\alpha > 0$ makes it unfeasible to define a suitable distance metric to solve Problem (P1). Consequently, these methods cannot be directly employed. Drawing inspiration from the aforementioned clustering algorithms, we introduce an algorithm to efficiently address the client partition in Problem (P1) in the following section.

\section{Hierarchical Clustering-based Collaborative Training Scheme}\label{sec:design}

\begin{figure*}[!t]
    \centering
    \includegraphics[width=\linewidth]{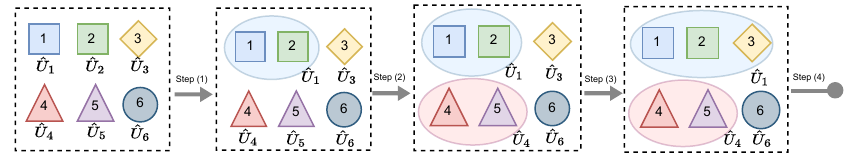}
    \caption{An illustration of the client partition process in HCCT with six clients, where the same shapes indicate similar data distributions. In the beginning, each client is a singleton group. In Step (1), clients $1$ and $2$ are clustered together since $(1,2)=\arg\max_{(k_1,k_2)} B(k_1,k_2)$. In Step (2), clients $4$ and $5$ are clustered together since $(4,5)=\arg\max_{(k_1,k_2)} B(k_1,k_2)$. In Step (3), group $1$ and client $3$ are clustered together since $(1,3)=\arg\max_{(k_1,k_2)} B(k_1,k_2)$.
    In Step (4), we stop partitioning since there is no benefit, i.e., $B(k_1, k_2)\leq 0, \forall k_1\neq k_2$. The final group number in this example is $3$.}
    \label{fig:demo}
\end{figure*}

In this section, we propose a \textit{hierarchical clustering-based collaborative training} scheme, coined HCCT, which generates a collaboration pattern to improve the generalization performance of all clients.
In HCCT, the training process is divided into $T$ training epochs.
In each training epoch, we partition the clients into $K$ groups and each group includes one or multiple clients.
After partition, clients train current models based on their local data.
If a group has only one client, this client performs independent training.
Otherwise, the clients in this group participate in collaborative training, with the server aggregating their model gradients.
We summarize the training process of HCCT in Algorithm \ref{algorithm} and detail the client partition step as follows.

\begin{algorithm}[tb]
\caption{Hierarchical Clustering-Based Collaborative Training (HCCT)}\label{algorithm}
\begin{algorithmic}[1] 
\STATE{Initialize $\mathbf{w}_{i}^{0}, \forall i=1,2,\dots,N$;}
\FOR{$t=0,1,\dots,T-1$}
    \IF{$t=0$}
        \STATE{Initialize $K=N$ groups $\mathcal{C}_k = \{k\}, \forall k\in[K]$;}
    \ELSE
        \STATE{Partition clients according to $\{\mathcal{C}_k\}_{k=1}^{K} \leftarrow \textit{ClientPartition} (\{\mathbf{g}_i^t\}_{i\in\mathcal{N}})$;}
    \ENDIF
    \FOR{each group $k=1,2,\dots,K$}
        \IF{ $|\mathcal{C}_k| = 1$ }
            \STATE{\textit{// Independent training}}
            \STATE{Client $i\in \mathcal{C}_k$ performs model training according to \eqref{eq:sgd} and uploads the gradient $\mathbf{g}_i^t$ to the server;}
            \STATE{Update local model as $\mathbf{w}_{i}^{t+1} = \mathbf{w}_{i}^{t} - \eta^t \mathbf{g}_i^t$;}
        \ELSE
            \STATE{\textit{// Collaborative training}}
            \STATE{Compute the dataset size as $\hat{D}_k = \sum_{i\in\mathcal{C}_k} D_i$;}
            \STATE{Compute the global model as $\hat{\mathbf{w}}_{k}^{t} = \sum_{i\in\mathcal{C}_k} \frac{D_i}{\hat{D}_k} \mathbf{w}_{i}^{t}$ and broadcast it to all clients in this group;}
            \FOR{each client $i\in \mathcal{C}_k$}
                \STATE{Perform local training according to \eqref{eq:sgd} and upload the model gradient $\mathbf{g}_i^t$ to the server;}
            \ENDFOR
            \STATE{The server updates the global model as $\hat{\mathbf{w}}_{k}^{t+1} = \hat{\mathbf{w}}_{k}^{t} - \eta^t \sum_{i\in\mathcal{C}_k} \frac{D_i}{\hat{D}_k} \mathbf{g}_i^t$ and broadcasts it to clients in $\mathcal{C}_k$;}
        \ENDIF
    \ENDFOR
\ENDFOR
\RETURN{$\mathbf{w}_{i}^{T}, \forall i\in\mathcal{N};$}
\end{algorithmic}
\end{algorithm}

\begin{algorithm}
\caption{Client Partition Function in HCCT}\label{algorithm:partition}
\begin{algorithmic}[1]
    \STATE{\textbf{def} \textit{ClientPartition}($\{\mathbf{g}_i^t\}_{i\in\mathcal{N}}, \{\mathcal{C}_k\}_{k=1}^{K}$):}
    \STATE{Initialize $K=N$ groups $\mathcal{C}_k = \{k\}$, and set $StopFlag$ as \textit{False}.}
    \WHILE{$StopFlag$ is \textit{False}}
        \FOR{group $k_1 = 1, 2,\dots, K-1$}
            \FOR{group $k_2 = k_1+1, 2,\dots, K$}
                \IF{$B(k_1,k_2)$ is not recorded}
                \STATE{Compute the averaged gradient $\mathbf{g}_{k^{\prime}}^{t}$ and dataset size $\hat{D}_{k^{\prime}}$ assuming groups $k_1$ and $k_2$ merge as group $k^\prime$;}
                \STATE{Record the benefit $B(k_1,k_2)$ as \eqref{eq:benefit};}
                \ENDIF
            \ENDFOR
        \ENDFOR
        \STATE{Find $k_1^*$ and $k_2^*$ that achieves the maximal utility according to \eqref{eq:max};}
        \STATE{Update the new group as $\mathcal{C}_{k_1^*} \leftarrow \mathcal{C}_{k_1^*} \cup \mathcal{C}_{k_2^*}$, remove group index $k_2^*$, and update the number of groups as $K\leftarrow K-1$;}
        \STATE{\textit{// Evaluate the stopping criterion}}
        \IF{$K\!=\!1$ or $B(k_1,k_2) \!\leq\! 0, \forall k_1,k_2 \!\in\! [K], k_1 \!\neq\! k_2$}
            \STATE{Set $StopFlag$ as \textit{True};}
        \ENDIF
    \ENDWHILE
\RETURN{$\{\mathcal{C}_k\}_{k=1}^{K}$;}
\end{algorithmic}
\end{algorithm}

We begin with $K=N$ groups each of which involves a single client.
In this case, the utility of each group equals the utility of its client, i.e.,
\begin{equation}
    \hat{U}_{k} = U_i \mathbb{I}\{i=k\}, \forall k\in[K],
\end{equation}
where $\mathbb{I}\{\cdot\}$ is the binary indicator function.
To improve the client utility, we need to decide whether to allow a coalition of any two groups.
Suppose groups $k_1$ and $k_2$ are merged as a new group $k^{\prime}$, i.e., $\mathcal{C}_{k^{\prime}} = \mathcal{C}_{k_1} \cup \mathcal{C}_{k_2}$, and thus the global gradient of group $k^{\prime}$ is a weighted average of the gradients from all clients in this group, i.e.,
\begin{equation}
    \hat{\mathbf{g}}_{k^{\prime}}^{t} = \sum_{i\in \mathcal{C}_{k^{\prime}}} \frac{D_i}{\hat{D}_{k^{\prime}}} \mathbf{g}_{i}^{t}.
\end{equation}
where $\hat{D}_{k^{\prime}} = \sum_{j\in \mathcal{C}_{k^{\prime}}} D_j$ is the number of training data samples of clients in $\mathcal{C}_{k^{\prime}}$.
For these clients in group $k^{\prime}$, we recompute their utilities and sum them up as the utility of this group, i.e.,
\begin{equation}
    \hat{U}_{k^{\prime}} = \sum_{i\in \mathcal{C}_{k^{\prime}}} U_i(\hat{D}_{k^{\prime}}, \hat{\mathbf{g}}_{k^{\prime}}^{t}).
\end{equation}
The \emph{benefit} of this merging step is defined as the difference of utilities after and before merging groups $k_1$ and $k_2$, i.e.,
\begin{equation}
    B(k_1,k_2) = \hat{U}_{k^{\prime}} - \hat{U}_{k_1} - \hat{U}_{k_2}.
    \label{eq:benefit}
\end{equation}
We perform the above operations iteratively on any two groups $k_1 \neq k_2, k_1,k_2\in [K]$ and record the corresponding benefits.
Subsequently, we merge two groups $k_1^*$ and $k_2^*$ that yield the maximal benefit, i.e.,
\begin{equation}
    (k_1^*, k_2^*) = {\arg\max}_{\{(k_1,k_2)|k_1 \neq k_2, k_1,k_2\in[K] \}} B(k_1,k_2).
    \label{eq:max}
\end{equation}
After this step, the optimization objective in (P1), i.e., the total utility of all clients, is increased since the utility of other groups remains unchanged.
In other words, the partitioning choice is optimal at the current stage. Now we have $K = N-1$ groups and repeat the above operations.

It is worth noting that there is no need to select beforehand the number of groups $K$.
In every partitioning step, the value of $K$ will decrease by one automatically.
Thus, the number of groups is implicitly determined by the stopping criterion.
Specifically, the process of client partition is terminated if there is only one group left, i.e., $K=1$, or there is no benefit in forming collaborations between any two groups, i.e.,
\begin{equation}
    B(k_1,k_2) \leq 0, \forall k_1,k_2\in [K], k_1 \neq k_2.
\end{equation}
Through multiple steps, HCCT builds a hierarchical representation of clients and finalizes the grouping among them.
The process of client partition is summarized in Algorithm \ref{algorithm:partition}, and we show an example with six clients in Fig. \ref{fig:demo}.

\subsection{Discussions}
We note that the above partition operations in Algorithm \ref{algorithm:partition} are carried out by the server and do not incur any additional cost for clients.
Moreover, HCCT is independent of group number determination or group center initialization, which is different from the traditional clustering methods for FL \cite{ifca,li2021soft}.
These advantages avoid the non-trivial parameter tuning process for different training tasks. 

In some scenarios, the result of client partition in HCCT can degenerate into the following collaboration patterns:
\begin{itemize}
    \item \textbf{Independent training}: each client performs independent learning without any cooperation with others, i.e., $K=N$ and $\mathcal{C}_i=\{i\}, \forall i\in\mathcal{N}$. Typical scenarios include: 1) clients have sufficient local training data, and 2) the data distributions diverge significantly among clients.
    \item \textbf{Global training} (conventional FL): all clients work together to train a global model, i.e., $K=1$ and $\mathcal{C}_1=\mathcal{N}$. Typical scenarios include: 1) clients have a very limited volume of training data, and 2) the data distributions are very similar or even IID among clients.
\end{itemize}
As will be verified via simulations, the proposed partition algorithm adapts to various scenarios by particularizing to the above collaboration patterns.

\textbf{Connection to personalized FL.}
The proposed HCCT scheme is orthogonal to the personalization techniques in FL \cite{huang2021personalized,yu2020salvaging,tan2022towards,arivazhagan2019federated,fallah2020personalized}.
Specifically, each group in HCCT can be viewed as ``a micro FL system'' that contains a set of clients.
Thus, we can employ any personalization technique to enhance the clients' performance in a group.
For example, the clients may only collaborate to train the feature extractor of a model and keep the classifier as a personalization layer locally.
This variant, named \textbf{HCCT-P} is advantageous when clients' data are more diverse, as incorporating personalized layers can enhance client-specific model accuracy. As such, HCCT-P is tailored to scenarios characterized by a high degree of non-IID data among clients.
In Section \ref{sec:compatibility}, we demonstrate the compatibility of HCCT with personalization techniques via simulations.

\textbf{Computational complexity.}
The proposed HCCT scheme incurs only the computation cost on the server without bringing any additional overhead to clients.
Specifically, the server needs to compute client similarity to evaluate the client's utility.
The computational complexity of client similarity is related to the number of clients and the number of parameters in the gradient, as detailed in the following.
On one hand, the first step of the client partition function requires computing gradient similarity between two clients, which is repeated for $\frac{1}{2}N(N-1)$ times.
In the best case, when there is no benefit for merging any two groups, the computational cost is $\mathcal{O}(N(N-1))$.
Otherwise, in each step $2\leq l\leq N$, there is a newly generated group, which consists of at most $l+1$ clients, and other $N-l-1$ groups.
Since we iterate over other groups and compute the benefits if it is merged with this new group, this process results in a computational cost of $\mathcal{O}(N(l+2)-(l+1)^2)$.
In the worst case, where we continuously put clients into one single group until it becomes a global training, the computational complexity is $\mathcal{O}(\frac{1}{6}(N^3-19N+6))$.
It is worth noting that the number of clients $N$ in cross-silo FL is rather small \cite{kairouz2021advances}, leading to an acceptable cost on the server.
On the other hand, since the gradient similarity is computed as the vector multiplication between two $M$-dimensional gradients, it introduces a computational complexity of $\mathcal{O}(M)$.

\subsection{Efficient Estimation of Client Similarity}\label{sec:efficient}
Computing client utility requires repeatedly evaluating client similarity $s\left( i, G_i \right)$.
Despite that this operation is performed on the central server, it still introduces high computational complexity of $\mathcal{O}(M)$ due to the vector multiplication between two $M$-dimensional gradients.
To reduce such complexity, we propose an efficient implementation by finding the most important layer accounting for client similarity.
After the first training iteration, the server collects the gradients from all clients.
The gradient of client $i$ can be expressed in a layer-wise manner as follows:
\begin{equation}
    \mathbf{g}_{i}^t = (\mathbf{g}_{i}^t[1], \mathbf{g}_{i}^t[2],\dots,\mathbf{g}_{i}^t[l],\dots,\mathbf{g}_{i}^t[L]),
\end{equation}
where $\mathbf{g}_{i}^t[l]$ denotes the parameters of the $l$-th layer.
Then the server finds the layer with the largest relative variance, i.e.,
\begin{equation}
    l^* = \arg\max_{l\in[L]} \frac{Var(\mathbf{g}_{1}^t[l],\mathbf{g}_{2}^t[l],\dots,\mathbf{g}_{N}^t[l])}{Mean(\mathbf{g}_{1}^t[l],\mathbf{g}_{2}^t[l],\dots,\mathbf{g}_{N}^t[l])}.
\end{equation}
Intuitively, the difference between clients can be identified by solely using layer $l^*$.
In the client partition process, the client similarity is efficiently computed as:
\begin{equation}
    s\left( i, G_i \right) = \cos \langle \mathbf{g}_i^t[l^*], \mathbf{g}_{G_i}^t[l^*] \rangle.
\end{equation}
This efficient implementation reduces the computational complexity from $\mathcal{O}(M)$ to $\mathcal{O}(dim(\mathbf{g}_i^t[l^*]))$.

To summarize, HCCT-E is an efficient implementation of HCCT, designed to reduce the computational complexity in the client partitioning process. This is particularly advantageous in environments where clients have similar data distributions and dataset sizes, as a coarse partition of clients is sufficient.

\section{Convergence Analysis}\label{sec:theory}

In this section, we analyze the proposed HCCT scheme to prove its convergence and observe the effect of client similarity.
To begin with, we make some commonly adopted assumptions on the local loss functions \cite{chen2020convergence,convergence1,convergence2,lin2022coexisting}.
It is worth noting that we do not assume any convexity of loss functions in this section. In other words, the following results hold for general non-convex loss functions.
For simplicity, we denote the local empirical loss function of client $i$ by:
\begin{equation}
    J_i(\mathbf{w}) \triangleq \frac{1}{D_i} \sum_{(\mathbf{x}, y)\in\mathcal{D}_i^{tr}} l( h((\mathbf{x};\mathbf{w}),y).
\end{equation}

\begin{assumption}\label{smooth}
($L$-smoothness)
There exists a constant $L>0$ such that for any $\mathbf{w}_1,\mathbf{w}_2\in\mathbb{R}^{M}$, we have:
\begin{equation}
    \| \nabla J_i(\mathbf{w}_1) - \nabla J_i(\mathbf{w}_2) \|_2 \leq L \| \mathbf{w}_1 - \mathbf{w}_2 \|_2, \forall i\in\mathcal{N}.
\end{equation}
\end{assumption}

\begin{assumption}\label{sgd}
(Unbiased and variance-bounded gradient)
On client $i$, the stochastic gradient computed on a random batch of data samples $\mathcal{B}$ is an unbiased estimate of the full-batch gradient over local training data $\mathcal{D}_i^{tr}$, i.e., 
\begin{equation}
    \mathbb{E} \left[ \frac{1}{|\mathcal{B}|} \sum_{(\mathbf{x}, y)\in\mathcal{B}} \nabla_{\mathbf{w}} l( h((\mathbf{x};\mathbf{w}),y)) \right] = \nabla J_i(\mathbf{w}), \forall i\in\mathcal{N}.
\end{equation}
Besides, there exists a constant $\sigma>0$ such that 
\begin{equation}
    \mathbb{E} \left[ \left\| \frac{1}{|\mathcal{B}|}\sum_{(\mathbf{x}, y)\in\mathcal{B}} \nabla_{\mathbf{w}} l( h((\mathbf{x};\mathbf{w}),y)) - \nabla J_i(\mathbf{w}) \right\|_2^2 \right] \leq \sigma^2, \forall i\in\mathcal{N}.
\end{equation}
\end{assumption}

The data distributions on clients vary among each other, as explained in Section \ref{sec:3A}, which leads to different optimization objectives and diverse gradients.
To show such data heterogeneity, previous works \cite{wang2022unified,sun2021semi} made assumptions on the difference between local and global training objectives.
Different from these works, we aim to investigate the collaboration among clients and thus make a more precise assumption on the gradient similarity among clients as follows.

\begin{assumption}\label{noniid}
(Gradient similarity)
There exist constants $\kappa_{i,j}>0$ such that
\begin{equation}
    \| \nabla J_i(\mathbf{w})- \nabla J_j(\mathbf{w}) \|_2 \leq \kappa_{i,j}, \forall i, j\in\mathcal{N}.
\end{equation}
\end{assumption}

We are now to derive the convergence of the HCCT in the following theorem.
Previous works in FL \cite{wang2021cooperative,mimic,po_fl} consider that an algorithm has reached convergence if it converges to a stationary point of the global loss function, i.e., if its expected squared gradient norm $\min_{t\in[T]} \mathbb{E}[\| \nabla J_i(\mathbf{w}) \|^2 ]$ is zero \cite{bottou2018optimization}.
Compared with these works, the clients in cross-silo FL have local training objectives and are not concerned about the global objective. 
Therefore, a better option is to verify the convergence of the algorithm by upper bounding the sum of squared norm of local gradients $\sum_{i\in\mathcal{N}} \mathbb{E}[\| \nabla J_i(\mathbf{w}_i^{t}) \|^2 ]$ \cite{partial}.

\begin{theorem}\label{thm:convergence}
Let $\xi_T=\sum_{t=0}^{T-1} \eta^t$. Consider client $i$ joins in group $G_i^t$ at epoch $t$.
With Assumptions \ref{smooth}-\ref{noniid}, if the learning rates satisfy $ \eta^t < \frac{1}{L}, \forall t\in[T]$, we have:
\begin{align}
    & \frac{1}{\xi_T} \sum_{t=0}^{T-1} \eta^t \sum_{i\in\mathcal{N}} \mathbb{E}[\| \nabla J_i(\mathbf{w}_i^{t}) \|^2 ] \nonumber\\
    \leq & \frac{2}{\xi_T} \sum_{i\in\mathcal{N}} \left(\mathbb{E}[J_i(\mathbf{w}_i^{0})] -  \mathbb{E}[J_i(\mathbf{w}_i^{*})]\right) \nonumber\\
    & + \frac{1}{\xi_T} \sum_{t=0}^{T-1} 2\eta^t \sum_{i\in\mathcal{N}} \sum_{j\in G_i^t} a_j^t \kappa_{i,j}^2 \nonumber\\
    & + \frac{1}{\xi_T} \sum_{t=0}^{T-1} 2 \eta^t L^2 \sum_{s=0}^{t-1} (\eta^s)^2 \sum_{i\in\mathcal{N}} \sum_{z\in\mathcal{C}_i^s} a_z^s \kappa_{i,z}^2 \nonumber\\
    & + \frac{1}{\xi_T} \sum_{t=0}^{T-1} L (\eta^t)^2 N \sigma^2 + \sum_{t=0}^{T-1} 4 \eta^t L^2 N\sum_{s=0}^{t-1} (\eta^s)^2 \sigma^2.
    \label{eq:thm2}
\end{align}
\end{theorem}
\begin{proof}
    We first prove an upper bound for the local loss decay of all clients at training epoch $t$ (formally stated in Lemma 1 in Appendix B), which is given by:
    \begin{align}
    & \sum_{i\in\mathcal{N}} \mathbb{E}[J_i(\mathbf{w}_i^{t+1})] - \sum_{i\in\mathcal{N}} \mathbb{E}[J_i(\mathbf{w}_i^{t})] \nonumber \\
    \leq & - \eta^t \sum_{i\in\mathcal{N}} \mathbb{E}\langle \nabla J_i(\mathbf{w}_i^t), \sum_{j\in G_i^t} a_j^t \nabla J_j(\mathbf{w}_j^t) \rangle \nonumber \\
    & + \frac{L (\eta^t)^2}{2} \sum_{i\in\mathcal{N}} \mathbb{E} \left[\left\| \sum_{j\in G_i^t} a_j^t \nabla J_j(\mathbf{w}_j^t) \right\|^2\right] + \frac{L (\eta^t)^2 N}{2} \sigma^2.
    \label{eq:lemma1-main}
    \end{align}
    
    In the right-hand side (RHS) of \eqref{eq:lemma1-main}, the inner product $- \eta^t \mathbb{E} \langle \nabla J_i(\mathbf{w}_i^{t}), \sum_{j\in G_i^t} a_j^t \nabla J_j(\mathbf{w}_j^{t}) \rangle$ needs to be further upper bounded.
    The main challenge is to bound the divergence between any two local gradients, i.e., $\mathbb{E}[\| \nabla J_i(\mathbf{w}_i^{t}) - \nabla J_i(\mathbf{w}_j^{t}) \|^2 ]$. Using $L$-smoothness in Assumption \ref{smooth}, we show that:
    \begin{equation}
        \mathbb{E}[\| \nabla J_i(\mathbf{w}_i^{t}) - \nabla J_i(\mathbf{w}_j^{t}) \|^2 ] \leq L^2 \mathbb{E}[\| \mathbf{w}_i^{t} - \mathbf{w}_j^{t} \|^2 ].
    \end{equation}
    Then we observe that, despite evolving from different groups, $\mathbf{w}_j^{t}$ and $\mathbf{w}_i^{t}$ have the same initialization $\mathbf{w}^{0}$. Therefore, we can upper bound their difference by comparing the accumulated gradients in previous $t-1$ iterations.
    
    Finally, the result is completed by summing up both sides of an upper bound of (27) over index $t=0,1,\dots,T-1$ and dividing them by $\sum_{t=0}^{T-1} \frac{\eta^t}{2}$.
    The detailed proof is deferred to Appendix C.
\end{proof}

\begin{corollary}\label{corollary}
    When the learning rates satisfy $\lim_{T\rightarrow \infty} \sum_{t=0}^{T-1} \eta^t = \infty$, $\lim_{T\rightarrow \infty} \sum_{t=0}^{T-1} (\eta^t)^{2} < \infty$, and $\eta^t=\mathcal{O} \left( (\sum_{i\in\mathcal{N}} \sum_{j\in G_i^t} a_j^t \kappa_{i,j}^2)^{1/p} \right)$ with $p>0$, the RHS of \eqref{eq:thm2} converges to zero as $T\rightarrow \infty$, i.e., the local models output by HCCT converge to the stationary points of the local loss functions.
\end{corollary}
\begin{proof}
    Please refer to Appendix D.
\end{proof}

From Theorem \ref{thm:convergence} and Corollary \ref{corollary}, we observe that the convergence of HCCT is hindered by the gradient dissimilarity among clients. Specifically, as more heterogeneous clients are grouped, i.e., $\sum_{j\in G_i^t} a_j^t \kappa_{i,j}^2$ is larger, the convergence speed will be slowed down.
This verifies our intuition that clients with similar gradients should be divided into the same collaboration group.
In the next section, we will show the empirical benefits of the proposed HCCT scheme via extensive simulations.

Note that we adopt a fixed learning rate during a local epoch and time-varying learning rates across (global) training epochs.
This setup follows previous studies on FL e.g., \cite{fedavg,wang2022unified}, and studies on local SGD, e.g., \cite{gu2022why,guohybrid}.
As we focus on comparing the impact of data distributions and dataset sizes, we adopt the simple but classic SGD optimizer as the client optimizer and leave the investigation of HCCT with varying local learning rates as future work.

\section{Simulation Results}\label{sec:simulation}

\subsection{Setup}\label{sec:setup}

We simulate a cross-silo FL system with one central server and $N$ clients. To comprehensively demonstrate the effect of various client data distributions, we simulate three scenarios with different training tasks:
\begin{itemize}
    \item \textbf{Digit} \cite{fedbn}: In this training task, $N=10$ clients collaborate to classify the digit images with labels ranging from $0$ to $9$. The task involves five datasets including SVHN \cite{svhn}, USPS \cite{usps}, SynthDigits \cite{SynthDigitsmnistm}, MNIST-M \cite{SynthDigitsmnistm}, and MNIST \cite{mnist}.
    Each client is assumed to have a set of randomly chosen training data sampled from one of these datasets.
    \item \textbf{FMNIST} \cite{fmnist}: In this training task, $N=20$ clients have IID data extracted from the FMNIST dataset. The number of data samples at each client follows a half-normal distribution $\mathcal{H}(1)$ \cite{duan2020self}.
    \item \textbf{CIFAR-10} \cite{cifar10}: In this training task, $N=10$ clients have heterogeneous and imbalanced data samples. 
    We divide the training data into 100 shards, each of which contains one random class of samples. Then, we randomly allocate a certain number of shards to each client using a half-normal distribution $\mathcal{H}(1)$.
\end{itemize}
For local evaluation, we split the data samples at each client into the training dataset and test dataset randomly.

In the Digit training task, we adopt the same CNN model with three convolutional layers as that of \cite{fedbn}.
For the FMNIST and CIFAR-10 training tasks, we train a fully connected neural network model and a CNN model with four convolutional layers, respectively.

For comparisons, we adopt the following training schemes as baselines:
\begin{itemize}
    \item \textbf{Independent training}: Each client performs local training independently based on the local training data.
    \item \textbf{Global training}: All clients collaboratively train a global model.
    \item \textbf{MAXFL} \cite{cho2022federate}: A client chooses to participate in training the global model only if its local loss is lower than a certain threshold; otherwise it retains independent training. 
    \item \textbf{FedFA}\footnote{This work \cite{huang2020fairness} adopts SGD with momentum as local training algorithm, but for fair comparisons we use mini-batch SGD in all methods.} \cite{huang2020fairness}: All clients train in a global model similar to the global training scheme. The key difference is that their model updates are weighted according to the training accuracy ${A}_i^t$ on local data, i.e., $\hat{\mathbf{g}}_{1}^{t} = \sum_{i\in\mathcal{N}} \frac{-\log_2 {A}_i^t}{ \sum_{j\in\mathcal{N}} -\log_2 {A}_j^t} \mathbf{g}_i^{t}$, which ensures equitable accuracy distribution among clients.
    \item \textbf{IFCA} \cite{ifca}: In every training epoch, each client selects the global model with the minimum loss on its local data from the given number of global models. 
    \item \textbf{FLSC} \cite{li2021soft}: In every training epoch, each client chooses $N_g$ global models with the minimum loss on their local data from multiple global models and then takes an average of these models as its local model.
\end{itemize}
IFCA and FLSC serve as benchmarks for client clustering in FL, while FedFA and MAXFL represent other approaches to determining client collaboration patterns.
To demonstrate the compatibility of HCCT, we also show the simulation results for two variants: HCCT-E (an efficient implementation in Section \ref{sec:efficient}) and HCCT-P (with personalization layers \cite{arivazhagan2019federated}).

We summarize the statistics of the above training tasks and the detailed experimental setup in Table \ref{table:setup}.
We run each experiment with five random seeds and report the average result.

\begin{table}[ht]
\caption{Simulation Setup}
\label{table:setup}
\centering
\resizebox{\columnwidth}{!}{
\begin{tabular}{cccc}
\toprule
                & Digit & FMNIST & CIFAR-10 \\ \midrule
\# of classes                 & 10    & 10       & 10      \\
\# of clients                 & 10    & 20       & 10     \\
Avg. \# of samples/client     &  185       & 120      &  950        \\
Batch size                    &  64     & 64        & 64        \\
Learning rate  $\eta^t$       & $0.1\times 0.995^{t-1}$      & $0.1\times 0.995^{t-1}$         &  $0.1\times 0.995^{t-1}$       \\
Local epochs                  & 5      &  1        & 1        \\
Training epochs               & 20      & 50         & 50        \\
\# of groups (IFCA, FLSC)         & 10      & 5         & 10        \\ 
\# of soft groups $N_g$ (FLSC) & 2      & 3         & 3 \\
$\alpha$ (HCCT)               & $1$    &  $100$       & $1$ \\
$\alpha$ (HCCT-E)               & $100$    &  $100$       & $10$ \\
$\alpha$ (HCCT-P)               & $1$    &  $100$       & $1$ \\ \bottomrule
\end{tabular}
}
\end{table}

\begin{table}[!t]
    \caption{Local test error (\%) in different training tasks. The best performances (except HCCT-P) are highlighted in \textbf{bold}.}
    \label{tab:three-dataset}
    \centering
    \resizebox{\columnwidth}{!}{
    \begin{tabular}{cccc}
        \toprule
         &  Digit&  FMNIST& CIFAR-10\\ \hline
         Indenpendant&  $30.22\pm 0.27$&  $48.38\pm 1.52$& $\mathbf{34.47\pm 4.73}$\\
         Global&  $29.55\pm 1.37$&  $\mathbf{29.98\pm 4.04}$& $46.96\pm 3.72$\\
         MAXFL&  $29.23\pm 1.56$&  $\mathbf{29.98\pm 4.09}$& $41.25\pm 6.99$\\
         FedFA&  $30.03\pm 1.38$&  $34.29\pm 3.88$& $44.54\pm 5.65$\\
         IFCA&  $28.63\pm 2.57$&  $38.84\pm 5.23$& $37.23\pm 2.96$\\
         FLSC&  $34.76\pm 6.26$&  $33.55\pm 4.16$& $41.11\pm 0.50$\\ \hline
        HCCT & $\mathbf{20.06\pm 0.66}$& $\mathbf{29.98\pm 1.37}$&$\mathbf{34.51\pm 4.64}$ \\
        HCCT-E & $24.05\pm 1.87$& $32.02\pm 2.55$&$\mathbf{34.51\pm 4.70}$ \\  \hline
        HCCT-P & $19.85\pm 0.31$& $29.51\pm 2.11$ & $34.37\pm 4.77$ \\ \bottomrule
    \end{tabular}}
\end{table}

\subsection{Performance Comparison}
We first compare the generalization performance under different training schemes in three training tasks.
Table \ref{tab:three-dataset} shows the average local test error of all clients after a given number of training epochs.
We observe that in the Digit training task, HCCT achieves minimal test error among all training schemes.
This is because HCCT finds an appropriate collaboration pattern for all clients which enhances their generalization performance by involving more training data while excluding gradients from clients with dissimilar data distribution.
In comparison, clients in independent training cannot achieve a good generalization performance because of limited training data, while global training suffers from poor convergence speed caused by heterogeneous data.
Moreover, other baseline schemes including MAXFL, FedFA, and FLSC, induce larger test errors at clients, since they cannot identify the intrinsic cluster structure of clients.
In addition, we find that in IFCA all clients may find the same model appealing and tend to join in that group, which degrades its generalization performance. 

In the FMNIST training task, HCCT and global training have similar test errors, and they outperform other baselines.
Given the IID data at clients, a global collaboration incorporates the maximum training data samples and significantly improves the generalization performance of the model.
However, clients in the CIFAR-10 task have heterogeneous data, and thus the global training scheme suffers from such non-IID issues.
In this case, clients should perform independent training, and HCCT also exhibits comparable test error to independent training.
These results evidence that HCCT can identify the most effective client collaboration pattern, including independent training and global training, to accommodate various scenarios.

To further compare different training schemes, we summarize the average, standard deviation, minimal, and maximal values of test errors at clients in the Digit training task in Table \ref{table:digit}.
We see that in addition to achieving the lowest average test error, HCCT also notably reduces the test error of the worst-performing client, which verifies that clients with fewer training data samples can benefit from effective collaboration.
Meanwhile, the generalization performances of those clients with sufficient training data are also preserved, as the proposed utility function excludes clients with dissimilar data from its collaboration group.
Besides, it is worth noting that HCCT achieves a lower standard deviation than the baseline schemes, which shows that HCCT implicitly enhances fairness among clients, leading to similar generalization performance for all clients.

\begin{table}[h]
\caption{Local test error (\%) in the Digit training task. The best performances are highlighted in \textbf{bold}.}
\label{table:digit}
\centering
\begin{tabular}{ccccc}
\toprule
 & Mean ($\downarrow$)& Std. ($\downarrow$)& Min. ($\downarrow$)& Max. ($\downarrow$)\\ \hline
Independent  & $30.22$     &  $18.89$    & $6.13$     & $56.01$     \\
Global       & $29.55$     &  $20.89$    & $4.99$     & $62.27$     \\
MAXFL        & $29.23$ & $20.83$ & $4.88$ & $62.01$ \\
FedFA        & $30.03$ & $21.10$ & $5.21$  & $63.71$ \\
IFCA         & $28.63$     &  $19.13$    & $4.70$     & $60.44$     \\
FLSC         & $34.76$     &  $21.39$    &  $6.49$    &  $67.99$    \\
\hline
HCCT         & $\mathbf{20.06}$     & $\mathbf{13.47}$     &  $\mathbf{4.19}$    & $\mathbf{39.44}$     \\ 
HCCT-E         & $24.05$     & $16.00$     &  $4.67$    & $45.77$     \\ \bottomrule
\end{tabular}
\end{table}

\subsection{Compatibility of HCCT}\label{sec:compatibility}

\textbf{Compatibility with personalization techniques.}
Table \ref{tab:three-dataset} also shows the local test error of HCCT with personalization layers (HCCT-P).
Compared with the original results of HCCT, HCCT-P further reduces the test error of clients.
However, since clients already found a proper collaboration pattern with others, the improvement is relatively limited.
A better combination with existing personalization techniques is left as future works.

\textbf{Flexibility of tackling incoming clients.}
To demonstrate the additional advantage of HCCT, we present the learning curves in Fig. \ref{fig:cifar_new} where 10 clients begin training and 0, 1, and 2 new clients join in training every 10 iterations.
Upon the arrival of a new client, its allocation to a specific group is determined based on the utility evaluation according to \eqref{eq:benefit} and \eqref{eq:max}. This immediate integration of the new client's data into the training process is evident from the observed increase in accuracy. Notably, the integration of additional clients over time contributes to a progressive reduction in the final training error. 
This demonstrates the flexibility and efficiency of our framework in not only integrating new clients but also in utilizing their data to improve the overall model performance.

\begin{figure}[!t]
    \centering
    \includegraphics[width=0.8\columnwidth]{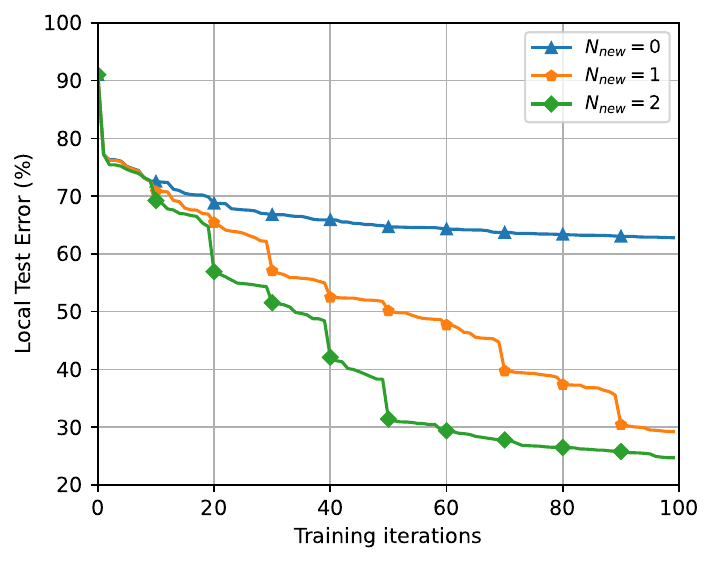}
    \caption{Local test error (mean) vs. training iterations in the CIFAR-10 dataset with $N_{new}$ new clients.}
    \label{fig:cifar_new}
\end{figure}

\subsection{Effect of System Parameters}\label{sec:effect}

\textbf{Effect of $\alpha$.}
In Fig. \ref{fig:cifar_alpha}, we investigate the effect of $\alpha$ using the CIFAR-10 dataset.
A larger value of $\alpha$ emphasizes the importance of data volume, while a small value means that clients tend to collaborate with someone with similar data distribution.
In this study, we split the assigned data for each client into the training data and test data according to different ratios, where a larger ratio means that a client has more training data samples but less test data.
In general, when $\alpha$ is small, its effect on the model performance is rather trivial.
We observe from the results that as the ratio becomes larger, more training data speed up the training process and enhance the generalization performance in all cases.
Moreover, if clients have fewer training data samples (i.e., the ratio is $0.05$ or $0.1$), the optimal value of $\alpha$ is $100$.
It means that clients have limited local data and tend to collaborate with others to benefit from their training data.
In contrast, when clients have many local data samples (i.e., ratio is $0.2$ or $0.5$), they will emphasize more on the data distribution.
Thus, by using a smaller $\alpha$ clients can avoid collaborating with other clients whose data distribution diverges significantly. 
These results provide a guideline for the selection of the value of $\alpha$ in the utility function.
In addition, it is straightforward to extend and apply personalized $\alpha$ for different clients.

\begin{figure}[!t]
    \centering
    \includegraphics[width=0.75\columnwidth]{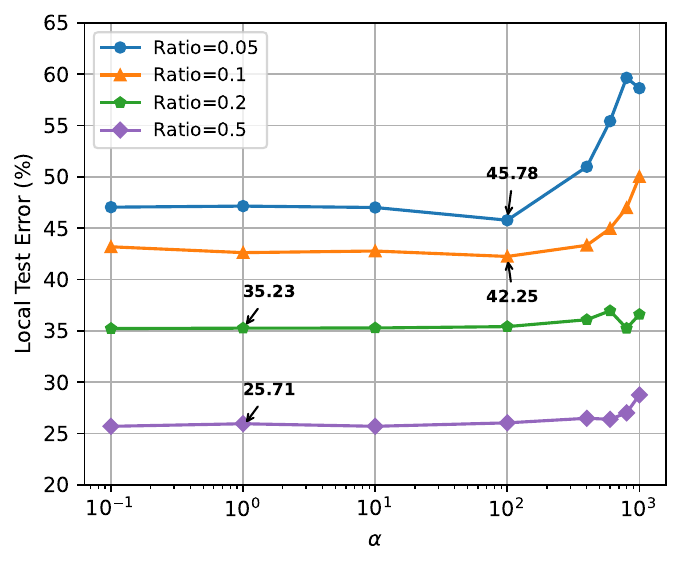}
    \caption{Local test error (mean) in the CIFAR-10 dataset with different values of $\alpha$.}
    \label{fig:cifar_alpha}
\end{figure}

\begin{figure}[!t]
    \centering
    \includegraphics[width=0.8\columnwidth]{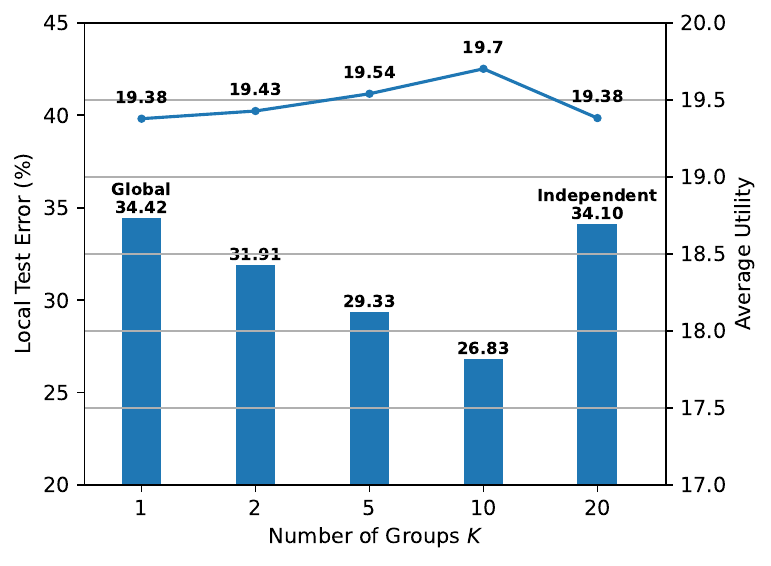}
    \caption{Local test error (mean) and utility in the Digit dataset with different numbers of groups.}
    \label{fig:cluster}
\end{figure}

\textbf{Effect of the number of groups.}
Next, to study how the number of groups affects the generalization performance, we set the number of clients as $N=20$, manually fix the number of groups as $K$, and evaluate the test error in the Digit training task.
From the results in Fig. \ref{fig:cluster}, we observe that in this scenario the optimal number of groups to achieve the minimal test error is $10$.
Notably, the maximal utility is also achieved when clients are divided into $10$ groups.
A plausible explanation is that each group consists of clients sharing the same dataset.
This implies that the utility function can effectively reflect the client’s requirement of improving the generalization performance, such that the minimal test error is achieved when the utility is maximized.

Furthermore, we set the number of clients as $N \in \{10,20,30,40,50\}$, manually fix the number of groups as $K\in\{1,2,5,10,N\}$, and evaluate the test error in the Digit training task.
The results in Fig. \ref{fig:optimal_K} show that the clients prefer collaborating with a limited number of clients instead of joining in a global federation.
\begin{figure}[!t]
    \centering
    \includegraphics[width=0.8\columnwidth]{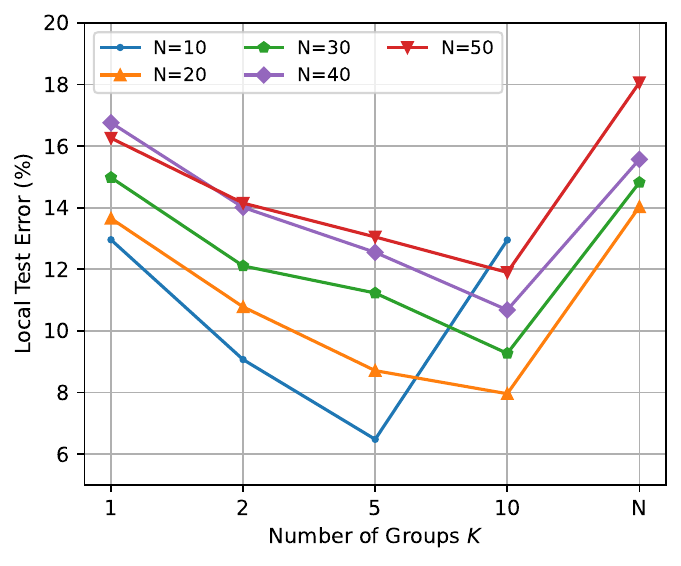}
    \caption{Local test error (mean) in the Digit dataset with different numbers of groups $K$ and clusters $N$.}
    \label{fig:optimal_K}
\end{figure}

\section{Conclusions}\label{sec:conclusion}
In this work, we tackled the challenge of optimizing the client collaboration pattern in order to maximize the generalization performance in cross-silo FL. We derived the generalization bound for clients in various collaboration cases and then formulated the client utility maximization problem. To efficiently solve this problem, we proposed HCCT, a hierarchical clustering-based collaborative training scheme, in which clients are partitioned into different non-overlapping groups without the need to initially decide the number of groups. We also proved the convergence of HCCT for general loss functions. The effectiveness of HCCT is further verified via extensive simulations in different scenarios and datasets.

The group partition function in HCCT requires evaluating client similarity for computing the client utility.
This introduces additional computational costs on the central server, increasing with the number of clients. 
Given this scalability issue, HCCT is limited to the cross-silo FL with fewer clients.
For future work, it is interesting to design the collaboration patterns for cross-device FL \cite{kairouz2021advances,mu2024federated}.
Besides, it is worth exploring personalization techniques in HCCT to further improve the generalization performance.
In addition, adapting HCCT to the case with time-varying local learning rates would be beneficial.

\bibliographystyle{IEEEtran}
\bibliography{ref}

\clearpage

\appendices
\section{Proof of Theorem 1}\label{proof:thm1}
Denote the true labeling functions on datasets $\mathcal{D}_i^{te}$ and $\hat{\mathcal{D}}_{G_i}^{tr}$ by $f_i(\cdot) \triangleq f_i(\cdot; \mathbf{w}_i^{*}): \mathcal{X}\rightarrow \mathcal{Y}$ and $f_{G_i}(\cdot) \triangleq f_{G_i}(\cdot; \hat{\mathbf{w}}_{G_i}^{*}): \mathcal{X}\rightarrow \mathcal{Y}$, which are characterized by optimal models $\mathbf{w}_i^{*}$ and $\hat{\mathbf{w}}_{G_i}^{*}$, respectively. 
In other words, they satisfy 
\begin{equation}
    y=f_i(\mathbf{x}; \mathbf{w}_i^{*}), \forall (\mathbf{x}, y) \in \mathcal{D}_i^{te},
\end{equation}
and 
\begin{equation}
    y=f_{G_i}(\mathbf{x}; \hat{\mathbf{w}}_{G_i}^{*}), \forall (\mathbf{x}, y) \in \hat{\mathcal{D}}_{G_i}^{tr}.
\end{equation}
Define two auxiliary risks as follows: 
\begin{equation}
    \epsilon_i(h, f_{G_i}) \triangleq \mathbb{E}_{\mathbf{x}\sim \mathcal{P}_{i}} [|h(\mathbf{x}) - f_{G_i}(\mathbf{x})|],
\end{equation}
\begin{equation}
    \epsilon_{G_i}(h, f_i) \triangleq \mathbb{E}_{\mathbf{x}\sim \hat{\mathcal{P}}_{G_i}} [|h(\mathbf{x}) - f_i(\mathbf{x})|].
\end{equation}

We begin with decomposing the test error $\epsilon_{i}(h)$ as follows:
\begin{align}
    & \epsilon_{i}(h) \nonumber \\
    \overset{(\text{a})}{=} & \epsilon_{i}(h) - \epsilon_{G_i}(h) + \epsilon_{G_i}(h) - \epsilon_{G_i}(h, f_i) + \epsilon_{G_i}(h, f_i) \nonumber \\
    \overset{(\text{b})}{\leq} & \epsilon_{G_i}(h) + \left| \epsilon_{G_i}(h, f_i) - \epsilon_{G_i}(h) \right| + \left| \epsilon_{i}(h, f_i) - \epsilon_{G_i}(h, f_i) \right| \nonumber \\
    \overset{(\text{c})}{\leq} & \epsilon_{G_i}(h) + \mathbb{E}_{\hat{\mathcal{P}}_{G_i}} \left[ f_{G_i}(\mathbf{x}) - f_{i}(\mathbf{x}) \right] + \left| \epsilon_{i}(h, f_i) - \epsilon_{G_i}(h, f_i) \right| \nonumber \\
    \overset{(\text{d})}{\leq} & \epsilon_{G_i}(h) + \mathbb{E}_{\hat{\mathcal{P}}_{G_i}} \left[ f_{G_i}(\mathbf{x}) - f_{i}(\mathbf{x}) \right] \nonumber \\
     &\quad + \int \left| \hat{\mathcal{P}}_{G_i} - \mathcal{P}_{i} \right| \left| h(\mathbf{x}) - f_{i}(\mathbf{x}) \right| d\mathbf{x} \nonumber \\
    \overset{(\text{e})}{\leq} & \epsilon_{G_i}(h) + \mathbb{E}_{\hat{\mathcal{P}}_{G_i}} \left[ f_{G_i}(\mathbf{x}) - f_{i}(\mathbf{x}) \right] + d_1\left(\hat{\mathcal{P}}_{G_i}, \mathcal{P}_{i}\right),
    \label{eq:help-1}
\end{align}
where (b) follows the fact $x+y\leq |x|+|y|, \forall x,y\in\mathbb{R}$. Besides, (c) and (d) apply the definition of several risks directly. Moreover, we obtain the result in (e) following \cite{ben2010theory}, which provides an upper bound for the error of a hypothesis on the target domain.

In (a), if we opt to add and subtract $\epsilon_i(h, f_{G_i})$ instead of $\epsilon_{G_i}(h, f_i)$, we arrive at a similar result of \eqref{eq:help-1} except substituting $\mathbb{E}_{\hat{\mathcal{P}}_{G_i}} \left[ f_{G_i}(\mathbf{x}) - f_{i}(\mathbf{x}) \right]$ with $\mathbb{E}_{\mathcal{P}_{i}} \left[ f_{G_i}(\mathbf{x}) - f_{i}(\mathbf{x}) \right]$. Therefore, by defining a constant $\lambda = \min \{\mathbb{E}_{\mathcal{P}_{i}} \left[ f_{G_i}(\mathbf{x}) - f_{i}(\mathbf{x}) \right], \mathbb{E}_{\hat{\mathcal{P}}_{G_i}} \left[ f_{G_i}(\mathbf{x}) - f_{i}(\mathbf{x}) \right] \}$, we obtain the following result:
\begin{align}
    \epsilon_{i}(h) \leq \epsilon_{G_i}(h)+ d_1\left(\hat{\mathcal{P}}_{G_i}, \mathcal{P}_{i}\right)+ \lambda.
\end{align}

Next, we provide an upper bound for the training error $\epsilon_{G_i}(h)$.
According to \cite[Theorem~2]{shalev2010learnability}, with probability at least $1-\delta$, we have
\begin{align}
    \epsilon_{G_i}(h)= \mathbb{E}_{\mathbf{x} \sim \hat{\mathcal{P}}_{G_i}} \left[\left|h_i(\mathbf{x})-f_{G_i}(\mathbf{x}) \right|\right] \leq \frac{4L^2}{\delta \mu \hat{D}_{G_i}},
    \label{eq:help-2}
\end{align}
where $\hat{D}_{G_i}$ is the size of the training dataset at cluster $G_i$.
Plugging \eqref{eq:help-2} into \eqref{eq:help-1} completes the proof.
\qed

\section{Additional Lemma}
In the following lemma, we analyze the local decay of all clients at training epoch $t$.
\begin{lemma}
Let $a_i^t \triangleq \frac{D_i}{\hat{D}_{G_i^t}}$.
With Assumptions 1-2, we have
    \begin{align}
    & \sum_{i\in\mathcal{N}} \mathbb{E}[J_i(\mathbf{w}_i^{t+1})] - \sum_{i\in\mathcal{N}} \mathbb{E}[J_i(\mathbf{w}_i^{t})] \nonumber \\
    \leq & - \eta^t \sum_{i\in\mathcal{N}} \mathbb{E}\langle \nabla J_i(\mathbf{w}_i^t), \sum_{j\in G_i^t} a_j^t \nabla J_j(\mathbf{w}_j^t) \rangle \nonumber \\
    & + \frac{L (\eta^t)^2}{2} \sum_{i\in\mathcal{N}} \mathbb{E} \big[\big\| \sum_{j\in G_i^t} a_j^t \nabla J_j(\mathbf{w}_j^t) \big\|^2\big] + \frac{L (\eta^t)^2 N}{2} \sigma^2.
    \label{eq:lemma1}
    \end{align}
\end{lemma}

\begin{proof}
Using Assumption 1, we have
\begin{align}
    & \mathbb{E}[J_i(\mathbf{w}_i^{t+1})] - \mathbb{E}[J_i(\mathbf{w}_i^{t})] \nonumber\\
    \leq & \mathbb{E}\langle \nabla J_i(\mathbf{w}_i^{t}), \mathbf{w}_i^{t+1} - \mathbf{w}_i^{t} \rangle + \frac{L}{2} \mathbb{E}[\left\| \mathbf{w}_i^{t+1} - \mathbf{w}_i^{t}\right\|^2] \nonumber\\
    \overset{(\text{a})}{=} & \mathbb{E}\langle \nabla J_i(\mathbf{w}_i^{t}), - \eta^t \sum_{j\in G_i^t} a_j^t \nabla J_j(\mathbf{w}_j^{t}) \rangle + \frac{L}{2} \mathbb{E}[\left\| \mathbf{w}_i^{t+1} - \mathbf{w}_i^{t}\right\|^2] \nonumber\\
    = & \mathbb{E}\langle \nabla J_i(\mathbf{w}_i^{t}), - \eta^t \sum_{j\in G_i^t} a_j^t \nabla J_j(\mathbf{w}_j^{t}) \rangle 
    + \frac{L (\eta^t)^2}{2} \mathbb{E}[\| \sum_{j\in G_i^t} a_j^t \mathbf{g}_j^{t} \|^2] \nonumber \\
    = & \mathbb{E}\langle \nabla J_i(\mathbf{w}_i^{t}), - \eta^t \sum_{j\in G_i^t} a_j^t \nabla J_j(\mathbf{w}_j^{t}) \rangle + \frac{L (\eta^t)^2}{2} \mathbb{E}[\| \sum_{j\in G_i^t} a_j^t \nabla J_j(\mathbf{w}_j^{t})\|^2 ] \nonumber\\
    & + \frac{L (\eta^t)^2}{2} \mathbb{E}[\| \sum_{j\in G_i^t} a_j^t \nabla J_j(\mathbf{w}_j^{t}) -  \sum_{j\in G_i^t} a_j^t \mathbf{g}_j^{t} \|^2]  \nonumber \\
    \overset{(\text{b})}{\leq} & - \eta^t \mathbb{E}\langle \nabla J_i(\mathbf{w}_i^{t}), \sum_{j\in G_i^t} a_j^t \nabla J_j(\mathbf{w}_j^{t}) \rangle \nonumber\\
    & + \frac{L (\eta^t)^2}{2} \mathbb{E} \Big[\Big\| \sum_{j\in G_i^t} a_j^t \nabla J_j(\mathbf{w}_j^{t}) \Big\|^2\Big] + \frac{L (\eta^t)^2}{2} \sigma^2,
\end{align}
where (a) follows the gradient unbiasedness and (b) follows the bounded gradient variance in Assumption 2.
    
\end{proof}

\section{Proof of Theorem 2}\label{proof:thm2}

\textbf{Upper bounding the inner product.}
We upper bound the inner product in \eqref{eq:lemma1} as follows:
\begin{align}
    & - \mathbb{E} \langle \nabla J_i(\mathbf{w}_i^{t}), \sum_{j\in G_i^t} a_j^t \nabla J_j(\mathbf{w}_j^{t}) \rangle \nonumber\\
    \overset{(\text{a})}{=} & - \frac{1}{2} \mathbb{E}[\| \nabla J_i(\mathbf{w}_i^{t}) \|^2 ] - \frac{1}{2} \mathbb{E}[\| \sum_{j\in G_i^t} a_j^t \nabla J_j(\mathbf{w}_j^{t}) \|^2 ] \nonumber \\
    & + \frac{1}{2} \mathbb{E}[\| \nabla J_i(\mathbf{w}_i^{t}) -\sum_{j\in G_i^t} a_j^t \nabla J_j(\mathbf{w}_j^{t}) \|^2 ] \nonumber\\
    \overset{(\text{b})}{=} & - \frac{1}{2} \mathbb{E}[\| \nabla J_i(\mathbf{w}_i^{t}) \|^2 ] - \frac{1}{2} \mathbb{E}[\| \sum_{j\in G_i^t} a_j^t \nabla J_j(\mathbf{w}_j^{t}) \|^2 ] \nonumber \\
    & + \frac{1}{2} \mathbb{E}[\| \nabla J_i(\mathbf{w}_i^{t}) - \nabla J_i(\mathbf{w}_j^{t}) + \nabla J_i(\mathbf{w}_j^{t}) -\sum_{j\in G_i^t} a_j^t \nabla J_j(\mathbf{w}_j^{t}) \|^2 ] \nonumber \\
    \overset{(\text{c})}{\leq} & - \frac{1}{2} \mathbb{E}[\| \nabla J_i(\mathbf{w}_i^{t}) \|^2 ] - \frac{1}{2} \mathbb{E}[\| \sum_{j\in G_i^t} a_j^t \nabla J_j(\mathbf{w}_j^{t}) \|^2 ] \nonumber\\
    & + \mathbb{E}[\| \nabla J_i(\mathbf{w}_i^{t}) - \nabla J_i(\mathbf{w}_j^{t}) \|^2 ] \nonumber\\
    & + \mathbb{E}[\|  \nabla J_i(\mathbf{w}_j^{t}) -\sum_{j\in G_i^t} a_j^t \nabla J_j(\mathbf{w}_j^{t}) \|^2 ] \nonumber\\
    \overset{(\text{d})}{\leq} & - \frac{1}{2} \mathbb{E}[\| \nabla J_i(\mathbf{w}_i^{t}) \|^2 ] - \frac{1}{2} \mathbb{E}[\| \sum_{j\in G_i^t} a_j^t \nabla J_j(\mathbf{w}_j^{t}) \|^2 ] \nonumber\\
    & + \mathbb{E}[\| \nabla J_i(\mathbf{w}_i^{t}) - \nabla J_i(\mathbf{w}_j^{t}) \|^2 ] + \sum_{j\in G_i^t} a_j^t \kappa_{i,j}^2,
    \label{eq:inner}
\end{align}
where (a) follows the fact $\langle \mathbf{x}, \mathbf{y} \rangle = \frac{1}{2} \|\mathbf{x}\|^2 + \frac{1}{2} \|\mathbf{y}\|^2 - \frac{1}{2} \|\mathbf{x} - \mathbf{y}\|^2$, in (b) we plus and minus $\nabla J_i(\mathbf{w}_j^{t})$ in the last term, (c) employs the inequality $\frac{1}{2} \|\mathbf{x} + \mathbf{y}\|^2 \leq \|\mathbf{x}\|^2 +\|\mathbf{y}\|^2$, and (d) follows Assumption 3. It is worth noting that by collaborating with similar clients, the last term $\sum_{j\in G_i^t} a_j^t \kappa_{i,j}^2$ can be mitigated due to their similar objectives.

For the third term in the RHS of \eqref{eq:inner}, we upper bound it as follows:
\begin{align}
    & \mathbb{E}[\| \nabla J_i(\mathbf{w}_i^{t+1}) - \nabla J_i(\mathbf{w}_j^{t}) \|^2 ] \nonumber \\
    \overset{(\text{e})}{\leq} & L^2 \mathbb{E}[\| \mathbf{w}_i^{t} - \mathbf{w}_j^{t} \|^2 ] \nonumber \\
    = & L^2 \mathbb{E}[\| \mathbf{w}^{0} - \sum_{s=0}^{t-1} \eta^s \sum_{z\in\mathcal{C}_i^s} a_z^s \mathbf{g}_z^s - \mathbf{w}^{0} + \sum_{s=0}^{t-1} \eta^s \sum_{z\in\mathcal{C}_j^s} a_z^s \mathbf{g}_z^s \|^2] \nonumber \\
    \overset{(\text{f})}{\leq} & L^2 \sum_{s=0}^{t-1} (\eta^s)^2 \mathbb{E}[\| \sum_{z\in\mathcal{C}_i^s} a_z^s (\mathbf{g}_z^s - \nabla J_i(\mathbf{w}_z^{s}) + \nabla J_i(\mathbf{w}_z^{s})) \nonumber \\
    & - \sum_{z\in\mathcal{C}_j^s} a_z^s (\mathbf{g}_z^s - \nabla J_j(\mathbf{w}_z^{s}) + \nabla J_j(\mathbf{w}_z^{s})) \|^2] \nonumber \\
    \overset{(\text{g})}{\leq} & L^2 \sum_{s=0}^{t-1} (\eta^s)^2 \mathbb{E}[\| \sum_{z\in\mathcal{C}_i^s} a_z^s ( \nabla J_i(\mathbf{w}_z^{s}) - \nabla J_j(\mathbf{w}_z^{s})) \|^2] \nonumber \\
    & + 2 L^2 \sum_{s=0}^{t-1} (\eta^s)^2 \sigma^2 \nonumber \\
    \overset{(\text{h})}{\leq} & L^2 \sum_{s=0}^{t-1} (\eta^s)^2 \sum_{z\in\mathcal{C}_i^s} a_z^s \kappa_{i,z}^2 + 2 L^2 \sum_{s=0}^{t-1} (\eta^s)^2 \sigma^2,\label{eq:39}
\end{align}
where (e) follows Assumption 1, and (f) follows the independence among different iterations. (g) and (h) hold due to Assumptions 2 and 3, respectively.

We plug the results of \eqref{eq:inner} and \eqref{eq:39} back into the RHS of \eqref{eq:lemma1} as follows:
\begin{align}
    & \sum_{i\in\mathcal{N}} \mathbb{E}[J_i(\mathbf{w}_i^{t})] - \sum_{i\in\mathcal{N}} \mathbb{E}[J_i(\mathbf{w}_i^{t})] \nonumber\\
    \leq & - \frac{\eta^t}{2} \sum_{i\in\mathcal{N}} \mathbb{E}[\| \nabla J_i(\mathbf{w}_i^{t}) \|^2 ] - \sum_{i\in\mathcal{N}} \frac{\eta^t}{2} \mathbb{E}[\| \sum_{j\in G_i^t} a_j^t \nabla J_j(\mathbf{w}_j^{t}) \|^2 ] \nonumber\\
    & + \eta^t \sum_{i\in\mathcal{N}} \sum_{j\in G_i^t} a_j^t \kappa_{i,j}^2
    + \frac{L (\eta^t)^2}{2} \sum_{i\in\mathcal{N}} \mathbb{E} [\| \sum_{j\in G_i^t} a_j^t \nabla J_j(\mathbf{w}_j^{t}) \|^2] \nonumber\\
    & + \frac{L (\eta^t)^2 N}{2} \sigma^2 + \eta^t \sum_{i\in\mathcal{N}} \mathbb{E}[\| \nabla J_i(\mathbf{w}_i^{t}) - \nabla J_i(\mathbf{w}_j^{t}) \|^2 ] \nonumber\\
    \leq & - \frac{\eta^t}{2} \sum_{i\in\mathcal{N}} \mathbb{E}[\| \nabla J_i(\mathbf{w}_i^{t}) \|^2 ] \nonumber\\
    & - (\frac{\eta^t}{2} - \frac{L (\eta^t)^2}{2}) \sum_{i\in\mathcal{N}} \mathbb{E}[\| \sum_{j\in G_i^t} a_j^t \nabla J_j(\mathbf{w}_j^{t}) \|^2 ] \nonumber\\
    & + \eta^t \sum_{i\in\mathcal{N}} \sum_{j\in G_i^t} a_j^t \kappa_{i,j}^2 \nonumber\\
    & + \eta^t L^2 \sum_{s=0}^{t-1} (\eta^s)^2 \sum_{i\in\mathcal{N}} \sum_{z\in\mathcal{C}_i^s} a_z^s \kappa_{i,z}^2 \nonumber\\
    & + \frac{L (\eta^t)^2 N}{2} \sigma^2 + 2 \eta^t L^2 N \sum_{s=0}^{t-1} (\eta^s)^2 \sigma^2.
    \label{eq:help-11}
\end{align}

\textbf{Completing the proof.}
Now we rearrange the terms and sum up both sides of \eqref{eq:help-11} over $ t=0,1,\dots,T-1$ to obtain the following result:
\begin{align}
    & \sum_{t=0}^{T-1} \frac{\eta^t}{2} \sum_{i\in\mathcal{N}} \mathbb{E}[\| \nabla J_i(\mathbf{w}_i^{t}) \|^2 ] \nonumber\\
    \leq & \sum_{i\in\mathcal{N}} \mathbb{E}[J_i(\mathbf{w}_i^{0})] - \sum_{i\in\mathcal{N}} \mathbb{E}[J_i(\mathbf{w}_i^{*})] \nonumber\\
    & + \sum_{t=0}^{T-1}  \eta^t \sum_{i\in\mathcal{N}} \sum_{j\in G_i^t} a_j^t \kappa_{i,j}^2 \nonumber\\
    & + \sum_{t=0}^{T-1} \eta^t L^2 \sum_{s=0}^{t-1} (\eta^s)^2 \sum_{i\in\mathcal{N}} \sum_{z\in\mathcal{C}_i^s} a_z^s \kappa_{i,z}^2 \nonumber\\
    & + \sum_{t=0}^{T-1} \frac{L (\eta^t)^2 N}{2} \sigma^2 + \sum_{t=0}^{T-1} 2 \eta^t L^2 N \sum_{s=0}^{t-1} (\eta^s)^2 \sigma^2,
    \label{eq:help-3}
\end{align}
where we use $- (\frac{\eta^t}{2} - \frac{L (\eta^t)^2}{2}) < 0$ and $- \mathbb{E}[J_i(\mathbf{w}_i^{T})] \leq - \mathbb{E}[J_i(\mathbf{w}_i^{*})]$.

By dividing both sides of \eqref{eq:help-3} over $\sum_{t=0}^{T-1} \frac{\eta^t}{2}$, we complete the proof of Theorem 2.
\qed

\section{Proof of Corollary 1}
Recall the inequality \eqref{eq:thm2} as:
\begin{align}
    & \frac{1}{\xi_T} \sum_{t=0}^{T-1} \eta^t \sum_{i\in\mathcal{N}} \mathbb{E}[\| \nabla J_i(\mathbf{w}_i^{t}) \|^2 ] \nonumber\\
    \leq & \frac{2}{\xi_T} \underbrace{\sum_{i\in\mathcal{N}} \left(\mathbb{E}[J_i(\mathbf{w}_i^{0})] -  \mathbb{E}[J_i(\mathbf{w}_i^{*})]\right)}_{S_1} \\
    & + \frac{1}{\xi_T} \underbrace{\sum_{t=0}^{T-1} \eta^t \sum_{i\in\mathcal{N}} \sum_{j\in G_i^t} a_j^t \kappa_{i,j}^2}_{S_2} \nonumber\\
    & + \frac{1}{\xi_T} \underbrace{\sum_{t=0}^{T-1} \eta^t L^2 \sum_{s=0}^{t-1} (\eta^s)^2 \sum_{i\in\mathcal{N}} \sum_{z\in\mathcal{C}_i^s} a_z^s \kappa_{i,z}^2}_{S_3} \\
    & + \frac{1}{\xi_T} \underbrace{\sum_{t=0}^{T-1} L (\eta^t)^2 N \sigma^2  + \sum_{t=0}^{T-1} \eta^t L^2 N \sum_{s=0}^{t-1} (\eta^s)^2 \sigma^2}_{S_4}.
\end{align}

If the learning rates satisfy $\lim_{T\rightarrow \infty} \xi_T = \lim_{T\rightarrow \infty} \sum_{t=0}^{T-1} \eta^t = \infty$, we have:
\begin{align}
    \lim_{T\rightarrow \infty} \frac{S_1}{\xi_T} = 0,
\end{align}
as $S_1$ is irrelevant of $t$.
Besides, the second term satisfies:
\begin{equation}
    \lim_{T\rightarrow \infty} \frac{S_2}{\xi_T} = \mathcal{O} \left(\frac{\sum_{t=0}^{T-1} \eta^t (\sum_{i\in\mathcal{N}} \sum_{j\in G_i^t} a_j^t \kappa_{i,j}^2)^p}{\sum_{t=0}^{T-1} \eta^t}\right) = 0,
\end{equation}
since $\eta^t=\mathcal{O} \left( (\sum_{i\in\mathcal{N}} \sum_{j\in G_i^t} a_j^t \kappa_{i,j}^2)^{1/p} \right)$ with $p>0$.
In addition, we have $\lim_{T\rightarrow \infty} \frac{S_3}{\xi_T} = 0$ and $\lim_{T\rightarrow \infty} \frac{S_4}{\xi_T} = 0$.
Thus, the RHS of \eqref{eq:thm2} converges to zero when $T\rightarrow \infty$.

\qed

\end{document}